\documentclass{article}
\usepackage[nonatbib, final]{neurips_2019}
\usepackage{amsmath,amssymb,amsthm}
\usepackage{algorithm}
\usepackage{algpseudocode}
\usepackage{subcaption}

\newtheorem{theorem}{Theorem}
\newtheorem{definition}{Definition}
\newtheorem{lemma}{Lemma}

\newtheorem{proposition}{Proposition}

\title{Stochastic Newton and Cubic Newton Methods with Simple Local Linear-Quadratic Rates}

\author{Dmitry Kovalev \\ KAUST\thanks{King Abdullah University of Science and Technology, Thuwal, Saudi Arabia} \\ \And Konstantin Mishchenko \\ KAUST \And Peter Richt\'arik \\ KAUST}

\usepackage[colorinlistoftodos,bordercolor=orange,backgroundcolor=orange!20,linecolor=orange,textsize=scriptsize]{todonotes}
\newcommand{\E}{\mathbb{E}}
\def\<#1,#2>{\left\langle #1,#2\right\rangle}
\newcommand{\R}{\mathbb{R}}
\newcommand{\cO}{\mathcal{O}}
\newcommand{\cW}{\mathcal{W}}
\newcommand{\cV}{\mathcal{V}}
\newcommand{\eqdef}{\stackrel{{\rm def}}{=}}
\newcommand{\sumin}{\sum_{i=1}^n}
\newcommand{\avein}{\frac{1}{n}\sum_{i=1}^n}

\newcommand{\norm}[1]{\left\| #1 \right\|}
\newcommand{\Ek}{\mathbb{E}_k}

\begin{document}

\maketitle
\begin{abstract}
We present two new remarkably simple stochastic second-order methods for minimizing the average of a very large number of sufficiently smooth and strongly convex functions. The first is a stochastic variant of Newton's method (SN), and the second is a stochastic variant of cubically regularized Newton's method (SCN). We establish local linear-quadratic convergence results. Unlike existing stochastic variants of second order methods, which require the evaluation of a large number of gradients and/or Hessians in each iteration to guarantee convergence, our methods do not have this shortcoming. For instance, the simplest variants of our methods in each iteration need to compute the gradient and Hessian of a {\em single} randomly selected function only. In contrast to most existing stochastic Newton and quasi-Newton methods, our approach guarantees local convergence faster than with first-order oracle and adapts to the problem's curvature. Interestingly, our method is not unbiased, so our theory provides new intuition for designing new stochastic methods. 
\end{abstract}

\section{Introduction}
The problem of empirical risk minimization (ERM) is ubiquitous in machine learning and data science. Advances of the last few years~\cite{roux2012stochastic, SAGA, SVRG, SDCA} have shown that availability of finite-sum structure and ability to use extra memory allow for methods that solve a wide range of ERM problems~\cite{nguyen2017sarah, nguyen2017stochastic, ryu2017proximal, allen2017katyusha, mishchenko2019stochastic, kovalev2019don}. In particular, Majorization-Minimization~\cite{lange2000optimization} (MM) approach, also known as successive upper-bound minimization~\cite{razaviyayn2016stochastic}, gained a lot of attention in first-order~\cite{mairal2013optimization, defazio2014finito, bietti2017stochastic, mokhtari2018surpassing, qian2019miso}, second-order~\cite{mokhtari2018iqn}, proximal~\cite{lin2015universal, ryu2017proximal} and other settings~\cite{karimi2018misso, mishchenko2018delay}. 

The aforementioned finite-sum structure is given by the  minimization formulation
\begin{equation}
\min \limits_{x\in \R ^{d}} \biggl[f( x)
	\eqdef \frac {1}{n}\sum \limits^{n}_{i=1}f_{i}( x)\biggr],\label{eq:pb}
\end{equation}
where each function $f_i:\R^d\to \R$ is assumed to have Lipschitz Hessian. 

Since in practical applications $n$ is often very large, this problem is typically solved using stochastic first order methods, such as Stochastic Gradient Descent (SGD) \cite{RobbinsMonro:1951}, which have cheap iterations independent of $n$.  However, constant-stepsize SGD provides fast convergence to a neighborhood of the solution only~\cite{sgd_and_hogwild, SGD-AS}, whose radius is proportional to the variance of the stochastic gradient at the optimum. So-called variance-reduced methods resolve this issue~\cite{roux2012stochastic, SVRG, sigma_k, GJS}, but their iteration complexity relies significantly on the \textit{condition number} of the problem, which is equal to the ratio of the smoothness and strong convexity parameters.

Second-order methods, in contrast, adapt to the curvature of the problem and thereby decrease  their dependence on the  condition number~\cite{NesterovBook}. Among the most well-known second-order methods are Newton's method (see e.g.~\cite{NesterovBook, karimireddy2018global, gower2019rsn}), a variant thereof with cubic regularization~\cite{nesterov2006cubic} and BFGS~\cite{liu1989limited, avriel2003nonlinear}. The vanilla deterministic Newton method  uses simple gradient update preconditioned via the inverse Hessian evaluated at the same iterate, i.e.,
\begin{align*}
	x^{k+1} = x^k - (\nabla^2 f(x^k))^{-1}\nabla f(x^k).
\end{align*}
This can be seen as approximating function $f$ locally around $x^k$ as\footnote{For $x\in \R^d$ and an $n\times n$ positive definite matrix $M$, we let $\left\|x\right\|_{M}\eqdef \left(x^\top M x\right)^{1/2}$.}
\begin{align}
	f(x)\approx f(x^k) + \<\nabla f(x^k), x - x^k> + \frac{1}{2}\left\|x-x^k\right\|^2_{\nabla^2 f(x^k)},
	\label{eq:newton_approximation}
\end{align}
and subsequently minimizing this approximation in $x$. However, this approximation is not necessarily an upper bound on $f$, which makes it hard to establish global rates for the method. Indeed, unless extra assumptions \cite{karimireddy2018global, gower2019rsn} or regularizations \cite{nesterov2006cubic, RBCN, doikov2019minimizing} are employed, the method can diverge. However, second order methods are famous for their fast (superlinear or quadratic) local convergence rates.

Unfortunately, the fast convergence rate of Newton and cubically regularized Newton methods  does not directly translate to the stochastic case. For instance, the quasi-Newton methods in~\cite{luo2016proximal, moritz2016linearly, gower2016stochastic} have complexity proportional to $\cO(\kappa^M)$ with some $M\gg 1$ instead of $\cO(\kappa)$ in~\cite{SVRG}, thus giving a worse theoretical guarantee. Other theoretical gaps include inconsistent assumptions such as bounded gradients together with strong convexity~\cite{berahas2016multi} .

Surprisingly, a lot of effort~\cite{kohler2017sub, bollapragada2018exact, zhou2018stochastic, zhang2018adaptive, tripuraneni2018stochastic, zhou2019stochastic} has been put into developing methods with massive (albeit sometimes still independent of $n$) batch sizes, typically proportional to $\cO(\epsilon^{-2})$, where $\epsilon$ is the target accuracy. This essentially removes all randomness from the analysis by making the variance infinitesimal. In fact, these batch sizes are likely to exceed $n$, so one ends up using mere deterministic algorithms. Furthermore, to make all estimates unbiased, large batch sizes are combined with preconditioning sampled gradients using the Hessians of \textit{some other} functions. We naturally expect that these techniques would not work in practice when only a single sample is used since the Hessian of a randomly sampled function might be significantly different from the Hessians of other samples.

{\em  Unlike all these methods, ours will work even with batch size equal one.} Thus, it is a major contribution of our work that we give up on taking unbiased estimates and show convergence despite the biased nature of our methods. This is achieved by developing new Lyapunov functions that are specific to second-order methods.

We are aware of two methods~\cite{rodomanov2016superlinearly, mokhtari2018iqn} that do not suffer from the issues above. Nevertheless, their update is cyclic rather than stochastic, which is known to be slower with the first-order oracle due to possibly bad choice of the iteration order~\cite{safran2019good}. 

\subsection{Basic definitions and notation}

Convergence of our algorithms will be established under certain regularity assumptions on the functions in~\eqref{eq:pb}. In particular, we will assume that all functions $f_i$ are twice differentiable, have a Lipschitz Hessian, and are strongly convex.  These concepts are defined next.   Let $\langle x,y\rangle$ and  $\|x\|\eqdef \langle x, x\rangle^{1/2}$ be the standard Euclidean inner product and norm in $ \R^d$, respectively. 

\begin{definition}[Strong convexity]
	A differentiable function $\phi: \mathbb{R}^d \rightarrow \mathbb{R}$ is $\mu$-strongly convex, where  $\mu>0$ is the strong convexity parameter, if for all $x, y \in \mathbb{R}^d$
	\begin{equation}\label{eq:bu9g9TF7bI}
		\phi( x) \geq \phi( y) +\langle  \nabla \phi( y) ,x-y\rangle  + \frac {\mu }{2}\left\| x-y\right\| ^{2}.
	\end{equation} 	

\end{definition}
For twice differentiable functions, strong convexity  is equivalent to requiring the smallest eigenvalue of the Hessian to be uniformly bonded above 0, i.e., $\nabla^2 \phi(x)\succeq \mu I$, where $I\in\R^{d\times d}$ is the identity matrix, or equivalently $ \lambda_{\min}(\nabla^2 \phi(x))\ge \mu$ for all $x\in \R^d$.   Given a twice differentiable strongly convex function $\phi: \R^d\to \R$, we let
$\|x\|_{\nabla^2 \phi(w)} \eqdef \langle  \nabla^2 \phi(w)x,  x\rangle^{\frac{1}{2}}$ be the norm induced by the matrix $\nabla^2 \phi(w)\succ 0$.

\begin{definition}[Lipschitz Hessian]\label{as:hess_smooth}
	Function $\phi: \mathbb{R}^d \rightarrow \mathbb{R}$ has $H$-Lipschitz Hessian if for all $x,y \in \mathbb{R}^d$
	\begin{equation*}
		\left\| \nabla ^{2} \phi( x) -\nabla ^{2} \phi \left( y\right) \right\| \leq H\| x-y\| .
	\end{equation*}
\end{definition}
Recall (e.g., see Lemma 1 in \cite{nesterov2006cubic}) that a function $\phi$ with $H$-Lipschitz Hessian admits for any $x,y\in\R^d$ the following estimates,
\begin{equation}\label{eq:buidgf79egf}
\left\| \nabla \phi(x) - \nabla \phi(y) - \nabla^2 \phi(y)(x-y) \right\| \leq \frac{H}{2} \|x-y\|^2,
% \quad \forall x,y\in \R^d,
\end{equation}
and
\begin{equation}\label{eq:taylor_inaccuracy_values}
	\bigl| \phi( x) - \phi(y) - \langle \nabla \phi(y), x-y \rangle - \frac{1}{2}\langle \nabla^2 \phi(y)(x-y), x-y \rangle \bigr| \leq \frac {H}{6}\left\| x-y\right\| ^{3}.
\end{equation}

Note that in the context of second order methods it is natural to work with functions with a Lipschitz Hessian. Indeed, this assumption is standard even for deterministic Newton \cite{NesterovBook} and cubically regularized Newton methods~\cite{nesterov2006cubic}.

\section{Stochastic Newton method}

In the big data case, i.e., when $n$ is very large, it is not efficient to evaluate the gradient, let alone the Hessian, of $f$ at each iteration. Instead, there is a desire to develop an incremental method which in each step resorts to computing the gradient and Hessian of a single function $f_i$ only. The challenge here is to design a scheme capable to use this stochastic first and second order information  to progressively throughout the iterations build more accurate approximations of the Newton step, so as to ensure a global convergence rate. While establishing {\em any} meaningful global rate is challenging enough, a fast linear rate would be desirable.

In this paper we propose to approximate the gradient and the Hessian using the latest information available, i.e.,
\[
	\nabla^2 f(x^k) \approx H^k \eqdef \frac{1}{n}\sum_{i=1}^n \nabla^2 f_i(w_i^k), \qquad \nabla f(x^k) \approx g^k = \avein \nabla f_i(w_i^k),
\]
where $w_i^k$ is the last vector for which the gradient $\nabla f_i$ and Hessian $\nabla^2 f_i$ was computed.  We then use these approximations to take a Newton-type step. Note, however, that $H^k$ and $g^k$ serve as good approximations only when used together, so that we precondition gradients with the Hessians at the same iterates. Although the true gradient, $\nabla f(x^k)$, gives precise information about the linear approximation of $f$ at $x^k$, preconditioning it with approximation $H^k$ would not make much sense. Our method is summarized in Algorithm~\ref{alg:newton}.

\begin{algorithm}[t]
\begin{algorithmic}
	\State \textbf{Initialize:} Choose starting iterates $w_1^0, w_2^0, \ldots, w_n^0 \in \R^d$ and minibatch size $\tau \in \{1,2,\dots,n\}$
	\For{$k = 0,1,2,\ldots$}
	\State $$x^{k+1} = \left[ \frac {1}{n}\sum ^{n}_{i=1}\nabla ^{2}f_{i}( w^{k}_{i}) \right] ^{-1}\left[\frac {1}{n}\sum ^{n}_{i=1}  \nabla^2 f_i(w_i^k) w_i^k-\nabla f_{i}( w^{k}_{i}) \right] $$
	\State Choose a subset $S^k \subseteq \{1,\ldots, n\}$ of size $\tau$ uniformly at random
	\State $w^{k+1}_{i}=\begin{cases}x^{k+1}, & \quad i\in S^{k} \\
	w^{k}_{i}, & \quad i\notin S^{k}
\end{cases}$ 
	\EndFor
\end{algorithmic}
\caption{Stochastic   Newton (SN)} \label{alg:newton}
\end{algorithm}

\subsection{Local convergence analysis}

We now state a simple local linear convergence  rate for  Algorithm~\ref{alg:newton}.   By $\E_k[\cdot]\eqdef \E[\cdot \mid x^k, w_1^k,\dotsc, w_n^k]$ we denote the expectation conditioned on all information prior to iteration $k+1$.

\begin{theorem}\label{thm:newton_minibatch}
Assume that every $f_i$ is $\mu$-strongly convex and has $H$-Lipschitz Hessian and consider the 
 following Lyapunov function:
 \begin{equation}
	\cW^k \eqdef \frac{1}{n} \sum_{i=1}^n
	\left\| w_i^k - x^{\star}\right\|^2.\label{eq:nb98dgf8f_new_lyap}
\end{equation}
 Then for the random iterates of Algorithm~\ref{alg:newton} we have the recursion
\begin{equation} \label{eq:u9g8gf9gu&&GjJBB}
	\mathbb{E}_k \left[ \cW^{k+1} \right] \leq 	\biggl( 1-\frac {\tau}{n}+ \frac{\tau}{n}\Bigl(\frac {H}{2\mu }\Bigr)^2 \cW^k \biggr) \cW^k.
\end{equation}
Furthermore, if $\|w_i^0 -x^{\star}\|\le \frac{\mu}{H}$ for $i=1,\dotsc, n$, then 
\begin{align*}
	\Ek\left[\cW^{k+1}\right]
	\le \left(1 - \frac{3\tau}{4n}\right)\cW^k.
\end{align*}
\end{theorem}
Clearly, when $\tau=n$, we achieve the standard superlinear (quadratic) rate of Newton method. When $\tau=1$, we obtain linear convergence rate that is independent of the conditioning, thus, the method provably adapts to the curvature.

\section{Stochastic Newton method with cubic regularization}

Motivated by the desire to develop a new principled variant of Newton's method that could enjoy global convergence guarantees, Nesterov and Polyak~\cite{nesterov2006cubic} proposed the following ``cubically regularized'' second-order approximation of $f$:
\begin{align*}
	f(x)\approx f(x^k) + \<\nabla f(x^k), x - x^k> + \frac{1}{2}\left\|x-x^k\right\|^2_{\nabla^2 f(x^k)} + \frac{M}{6}\left\|x-x^k\right\|^3.
\end{align*}
In contrast to \eqref{eq:newton_approximation}, the above {\em is} a global upper  bound on $f$ provided that  $M \geq H$. This fact is then used  to establish local rates.

\subsection{New method}

In this section  we combine cubic regularization with the stochastic Newton technique developed in previous section. Let us define the following second-order Taylor approximation of $f_i$:
\begin{equation*}
	\phi ^{k}_{i}( x) \eqdef f_{i}( w^{k}_{i}) + \left\langle  \nabla f_{i}(w^{k}_{i}) ,x-w^{k}_{i} \right\rangle  +\frac {1}{2}\left\| x-w^{k}_{i}\right\| _{\nabla ^{2}f_{i}( w_{i}^k) }^2.
\end{equation*}
Having introduced $\phi_i^k$, we are now ready to present our Stochastic  Cubic Newton method (Algorithm~\ref{alg:cubic}).

\begin{algorithm}[t]
\begin{algorithmic}
\State \textbf{Initialize:} Choose starting iterates $w_1^0, w_2^0, \ldots, w_n^0 \in \R^d$ and minibatch size $\tau \in \{1,2,\dots,n\}$
\For{$k = 0,1,2,\ldots$}
	\State $x^{k+1} = \arg\min\limits_{x \in \mathbb{R}^d} \frac{1}{n} \sum_{i=1}^n 
	\left[\phi_i^k(x) + \frac{M}{6} \left\| x - w_i^k \right\|^3 \right]$
	\State Choose a subset $S^k \subseteq \{1,\ldots, n\}$ of size $\tau$ uniformly at random
	\State $w^{k+1}_{i}=\begin{cases}x^{k+1}, & \quad i\in S^{k} \\
	w^{k}_{i}, & \quad i\notin S^{k}
\end{cases}$ 
\EndFor	
\end{algorithmic}
\caption{Stochastic  Cubic Newton (SCN)}\label{alg:cubic}
\end{algorithm}

\subsection{Local convergence analysis}

Our main theorem gives a bound on the expected improvement of a new Lyapunov function $\cV^k$.

\begin{theorem}\label{thm:cubic_new} Let $f$ be $\mu$-strongly convex and assume every $f_i$ has $H$-Lipschitz Hessian and consider the 
 following Lyapunov function: 
 \begin{equation} \label{eq:V_k} \cV^k\eqdef \frac{1}{n} \sum ^{n}_{i=1}\left( f( w^{k}_{i}) -f(x^{\star })\right) ^{\frac {3}{2}}.\end{equation}
Then for the random iterates of Algorithm~\ref{alg:cubic} we have the recursion
\[\E_k\left[ \cV ^{k+1}\right]  \leq \biggl(1 - \frac{\tau}{n} + \frac{\tau}{n}  \Bigl( \frac {\left( M+H\right) \sqrt {2}}{3\mu ^{\frac {3}{2}}}  \Bigr)^{3/2} \sqrt{\cV^k} \biggr) \cV^k. \]
Furthermore, if $f(w_i^0)-f(x^{\star})\le \frac{2\mu^3}{(M+H)^2}$ for $i=1,\dotsc, n$, then 
\begin{align*}
	\Ek\left[\cV^{k+1}\right]
	\le \left(1 - \frac{\tau}{2n}\right)\cV^k.
\end{align*}
\end{theorem}
Again, with $\tau=n$ we recover local superlinear convergence of cubic Newton. In addition, this rate is locally independent of the problem conditioning, showing that we indeed benefit from second-order information.

\clearpage
\bibliographystyle{plain}
\bibliography{stoch_newton_ref}

\clearpage
\appendix
\part*{Supplementary Material}

\section{Proofs for Algorithm~\ref{alg:newton}}

\subsection{Three lemmas}

\begin{lemma}\label{lem:distance_recurrence}
	Let  $f_i$ be $\mu$-strongly convex and have $H$-Lipschitz Hessian for all $i=1,\dotsc, n$. Then the iterates of Algorithm~\ref{alg:newton} satisfy
	\begin{equation}
		\left\| x^{k+1}-x^{\star }\right\| \leq \frac {H}{2\mu } \cdot \cW^k.\label{eq:b9fg8f8ve98bs09}
	\end{equation}
\end{lemma}
\begin{proof}
Letting $H^k \eqdef  \frac {1}{n}\sum ^{n}_{i=1}\nabla ^{2}f_{i}( w^{k}_{i})$, one step of the method can be written in the form
	\begin{align*}
		x^{k+1}= \left(H^k\right)^{-1}\left[ \frac {1}{n}\sum ^{n}_{i=1}\nabla ^{2}f_{i}( w^{k}_{i}) w^{k}_{i}-\frac {1}{n}\sum ^{n}_{i=1}\nabla f_{i}( w^{k}_{i}) \right].
	\end{align*}
	Subtracting $x^{\star}= \left(H^{k}\right)^{-1} H^k x^{\star}$ from both sides, and using the fact that $\sum_{i=1}^n \nabla f_i(x^{\star})=0$, this further leads to
	\begin{equation}\label{eq:n98fg98dd}
		x^{k+1} - x^{\star}
		= \left(H^k\right)^{-1} \frac {1}{n}\sum ^{n}_{i=1}\left[\nabla ^{2}f_{i}( w^{k}_{i}) \left( w^{k}_{i}-x^{\star }\right) -\left( \nabla f_{i}( w^{k}_{i}) -\nabla f_{i}( x^{\star }) \right)\right].
	\end{equation}
	
Next, note that since $f_i$ is $\mu$ strongly convex,  we have $\nabla^2 f_i(w_i^k) \succeq \mu I$ for all $i$. It implies that $H^k \succeq \mu I$, which in turns gives \begin{equation} \label{eq:b89fgff}\| \left(H^k\right)^{-1} \| \leq \frac{1}{\mu} .\end{equation}
The rest  of the proof follows similar reasoning as  in the standard convergence proof of Newton's method, with only small modifications:		
\begin{eqnarray*}
	\left\| x^{k+1}-x^{\star }\right\|
	& \overset{\eqref{eq:n98fg98dd}}{\leq} & \left\|  \left(H^k\right)^{-1} \right\| \left\| \frac {1}{n}\sum ^{n}_{i=1}\left[\nabla ^{2}f_{i}( w^{k}_{i}) \left( w^{k}_{i}-x^{\star }\right) -\left( \nabla f_{i}( w^{k}_{i}) -\nabla f_{i}( x^{\star }) \right)\right] \right\| \\
	&\overset{\eqref{eq:b89fgff}}{\leq} & \frac{1}{\mu} \left\| \frac {1}{n}\sum ^{n}_{i=1}    \left[\nabla f_{i}( x^{\star }) -  \nabla f_{i}( w^{k}_{i}) - \nabla ^{2}f_{i}( w^{k}_{i}) \left( x^{\star }- w^{k}_{i}\right)\right] \right\|  \\
	&\leq & \frac {1}{\mu n}\sum ^{n}_{i=1}\left\|  \nabla f_{i}( x^{\star }) -  \nabla f_{i}( w^{k}_{i}) - \nabla ^{2}f_{i}( w^{k}_{i}) \left( x^{\star }- w^{k}_{i}\right)  \right\| \\
	&\overset{\eqref{eq:buidgf79egf}}{\leq} & \frac {1}{\mu n}\sum ^{n}_{i=1}\frac {H}{2}\left\| w^{k}_{i}-x^{\star }\right\| ^{2}.
\end{eqnarray*}	
\end{proof}

\begin{lemma}
	Assume that each $f_i$ is $\mu$-strongly convex and has $H$-Lipschitz Hessian. If $\|w_i^0-x^{\star}\|\le \frac{\mu}{H}$ for $i=1,\dotsc,n$, then for all $k$,
	\begin{align}
		\cW^k\le \frac{\mu^2}{H^2}. \label{eq:newton_almost_sure}
	\end{align}
\end{lemma}
\begin{proof}

	We claim  that
	\begin{equation}\label{eq:induction_proof-hypothesis}
		\norm{w_i^k - x^{\star}}^2 \le \frac{\mu^2}{H^2}, \qquad \forall i \in \{1,2,\dots,n\}.
	\end{equation}
	The upper bound on $\cW^k$ then follows immediately.  

	 	We will  now prove the claim by induction in $k$. The statement \eqref{eq:induction_proof-hypothesis} holds for $k=0$  by our assumption. Assume it holds for $k$ and let us show that it holds for $k+1$. If $i\not\in S^k$, $w_i^k$ is not updated and inequality \eqref{eq:induction_proof-hypothesis} holds for $w_i^{k+1} = w_i^{k}$  by induction assumption. On the other hand, for every $i\in S^k$, we have
	\begin{eqnarray*}
		\norm{w_i^{k+1} - x^{\star}}
		= \norm{ x^{k}-x^{\star } }
		\overset{\eqref{eq:b9fg8f8ve98bs09}}{\le} \frac {H}{2\mu } \cdot \frac{1}{n}\sum ^{n}_{j=1} \norm{ w^{k}_{j}-x^{\star } }^{2} 
		\le \frac{H}{2\mu} \frac{1}{n}\sum_{j=1}^n \frac{\mu^2}{H^2}
< \frac{\mu}{H}.
	\end{eqnarray*}
	Thus, again, inequality \eqref{eq:induction_proof-hypothesis} holds for $w_i^{k+1}$.

\end{proof}

\begin{lemma}\label{lem:dual_recurrence_minibatch}
The random iterates of Algorithms~\ref{alg:newton} and~\ref{alg:cubic} satisfy the identity
\begin{equation}\label{eq:nbi7fgf9gybffxyz}
	\E_k \left[ \cW^{k+1} \right]= \frac{\tau}{n} \E_k\left[\left\| x^{k+1}-x^{\star }\right\| ^{2}\right] + \left( 1-\frac {\tau}{n}\right) \cW^k.
\end{equation}	
\end{lemma}
\begin{proof}
	For each $i$, $w_i^{k+1}$ is equal  to $x^{k+1}$ with probability $\frac{\tau}{n}$  and to $w_i^k$ with probability $1 - \frac{\tau}{n}$, hence the result.
\end{proof}

	\subsection{Proof of Theorem~\ref{thm:newton_minibatch}}

\begin{proof}
Using Lemma~\ref{lem:distance_recurrence} and Lemma~\ref{lem:dual_recurrence_minibatch}, we obtain
\begin{eqnarray*}
	\E_k\left[ \cW ^{k+1}\right]
	& \overset{\eqref{eq:nbi7fgf9gybffxyz}}{=} &  \frac{\tau}{n} \left\| x^{k+1}-x^{\star }\right\| ^{2}  +  \left( 1-\frac {\tau}{n}\right) \cW^k \\
	&\overset{\eqref{eq:b9fg8f8ve98bs09}}{\leq } &  \frac{\tau}{n}
	\left( \frac {H}{2\mu }\right)^2 \left(\cW^k \right)^2 + \left( 1-\frac {\tau}{n}\right)\cW^k \\	
	&= &	\left( 1-\frac {\tau}{n}+ \frac{\tau}{n}\left(\frac {H}{2\mu }\right)^2 \cW^k \right) \cW^k \\
	&\overset{\eqref{eq:newton_almost_sure}}{\le} & \left( 1-\frac {3\tau}{4n}\right)\cW^k,
\end{eqnarray*}
where in the last step we have also used the assumption on $\|w_i^0-x^{\star}\|$ for $i=1,\dotsc,n$.
\end{proof}

\section{Proofs for Algorithm~\ref{alg:cubic}}

\subsection{Four lemmas}

\begin{lemma}\label{lem:upper_bound_f_k_plus_one}
	Let each $f_i$ have $H$-Lipschitz Hessian and choose $M\ge H$. Then,
	\begin{equation}
		f(x^{k+1}) \leq \min\limits_{x \in \mathbb{R}^d} \left[ f(x) + \frac{M + H}{6} \cdot \frac{1}{n}\sum_{i=1}^n \left \|x - w_i^k \right\|^3 \right].\label{eq:bidg78fgvyd7820864}
	\end{equation}

\end{lemma}

\begin{proof}
	We begin by establishing an  upper bound for  $f$: 
	\begin{eqnarray*}
		f(x) \overset{\eqref{eq:pb}}{=} \frac{1}{n} \sum_{i=1}^n f_i(x) \overset{\eqref{eq:taylor_inaccuracy_values}}{\leq} \frac {1}{n}\sum ^{n}_{i=1}\left[\phi ^{k}_{i}( x) +\frac {H}{6}\left\| x-w^{k}_{i}\right\| ^{3} \right]
		\leq \frac {1}{n}\sum ^{n}_{i=1}\left[ \phi ^{k}_{i}( x) +\frac {M}{6} \left\| x-w^{k}_{i} \right\| ^{3}\right] .
	\end{eqnarray*}
	Evaluating both sides at $x^{k+1}$ readily gives
	\begin{eqnarray}
		f(x^{k+1})
		& \le & \frac {1}{n}\sum ^{n}_{i=1}\left[ \phi ^{k}_{i}( x^{k+1}) +\frac {M}{6} \left\| x^{k+1}-w^{k}_{i} \right\| ^{3}\right] \notag 
		\\
		& = & \min _{x\in \mathbb{R} ^{d}}\frac {1}{n}\sum ^{n}_{i=1}\left[ \phi ^{k}_{i}( x) +\frac {M}{6}\left\| x-w^{k}_{i}\right\| ^{3}\right], \label{eq:nb9f8geh78v87df}
	\end{eqnarray}
	where the last equation follows from the definition of $x^{k+1}$.
	We can further upper bound every $\phi_i^k$ using $\eqref{eq:taylor_inaccuracy_values}$ to obtain
	\begin{eqnarray}
			\frac {1}{n}\sum ^{n}_{i=1}\left[ \phi ^{k}_{i}( x) +\frac {M}{6} \left\| x-w^{k}_{i} \right\| ^{3}\right] 
			&\overset{\eqref{eq:taylor_inaccuracy_values}}{\leq} & \frac {1}{n}\sum ^{n}_{i=1}\left[ f_{i}( x) +\frac {M+H}{6}\left\| x-w^{k}_{i}\right\| ^{3}\right] \notag \\
			&=& f( x) +\frac {M+H}{6} \cdot \frac{1}{n} \sum ^{n}_{i=1} \left\| x-w^{k}_{i}\right\| ^{3}.\label{eq:0hd9h9df}
	\end{eqnarray}	
Finally, by combining \eqref{eq:nb9f8geh78v87df} and \eqref{eq:0hd9h9df}, we get
	\begin{align*}
		f( x^{k+1})
	&\leq \min _{x\in \mathbb{R} ^{d}}\left[ f( x) +\frac {M+H}{6} \frac{1}{n} \sum ^{n}_{i=1}\left\| x-w^{k}_{i}\right\| ^{3}\right] .
	\end{align*}

\end{proof}

\begin{lemma}
	Let $f$ be $\mu$-strongly convex, $f_i$ have $H$-Lipschitz Hessian for every $i$ and choose $M\ge H$.  Then
	\begin{equation}\label{eq:JuI98Gt67Y}
		f( x^{k+1}) -f(x^{\star }) \leq \frac {\sqrt{2}\left( M+H\right)}{3\mu ^{\frac {3}{2}}}  \cV^k.
	\end{equation}
\end{lemma}

\begin{proof}
The result readily follows from combining  Lemma~\ref{lem:upper_bound_f_k_plus_one} and strong convexity of $f$. Indeed, we have
	\begin{eqnarray*}
		f( x^{k+1})
		&\overset{\eqref{eq:bidg78fgvyd7820864}}{\leq} & \min _{x\in \mathbb{R} ^{d}}\left[ f( x) +\frac {M+H}{6n}\sum ^{n}_{i=1}\left\| x-w^{k}_{i}\right\| ^{3}\right] \\
		&\leq & f(x^{\star })+\frac {M+H}{6n}\sum ^{n}_{i=1}\left\| x^{\star}-w^{k}_{i}\right\| ^{3} \\
	&\overset{\eqref{eq:bu9g9TF7bI}}{\leq} & f(x^{\star }) +\frac {M+H}{6n}\sum ^{n}_{i=1}\left[ \frac {2}{\mu }\left( f( w^{k}_{i}) -f(x^{\star })\right) \right] ^{\frac {3}{2}}\\
&= & f(x^{\star }) +\frac {\left( M+H\right) \sqrt {2}}{3\mu ^{\frac {3}{2}}} \cdot \frac{1}{n} \sum ^{n}_{i=1}\left( f( w^{k}_{i}) -f(x^{\star })\right) ^{\frac {3}{2}}.
	\end{eqnarray*}
\end{proof}

\begin{lemma}
	Let $f$ be $\mu$-strongly convex, $f_i$ have $H$-Lipschitz Hessian for every $i$. If $M\ge H$ and $f(w_i^0)-f(x^{\star})\le \frac{2\mu^3}{(M+H)^2}$ for $i=1,\dotsc, n$, then 
	\begin{align}
		\cV^k \le \frac{27\mu^{\frac{9}{2}}}{8\sqrt{2}(M+H)^3}\label{eq:cubic_almost_sure}
	\end{align}
	holds with probability 1.
\end{lemma}
\begin{proof}
	Denote for convenience $C\eqdef \frac{\sqrt{2}(M+H)}{3\mu^{3/2}}$. The lemma's claim follows as a corollary of a stronger statement, that for any $i$ and $k$ we have with probability 1
	\begin{align*}
		(f(w_i^k) - f(x^{\star}))^{\frac{3}{2}}
		\le \frac{1}{4C^3}.
	\end{align*}
	Let us show it by induction in $k$, where for $k=0$ it is satisfied by our assumptions. If $i\not\in S^k$, $w_i^{k+1}=w_i^k$ and the induction assumption gives the step. Let $i\in S^k$ and, then
	\begin{eqnarray*}
		f(w_i^{k+1}) - f(x^{\star})
		= f(x^{k+1}) - f(x^{\star})
		&\overset{\eqref{eq:JuI98Gt67Y}}{\le} & \frac {C}{n}\cdot \sumin\left( f( w^{k}_{i}) -f(x^{\star })\right) ^{\frac {3}{2}}  \\
		&\le & \frac {C}{n}  \sumin \frac{1}{4C^3} \\
		&= & \frac{1}{4C^2} < \frac{1}{4^{2/3}C^2},
	\end{eqnarray*}
	so $(f(w_i^{k+1})-f(x^{\star}))^{3/2}\le \frac{1}{4C^3}$.
\end{proof}

\begin{lemma}\label{lem:bof98gHHVBY09}
The following equality holds for all $k$
\begin{equation}\label{eq:nb97fg9vb98h8s}
\mathbb{E}_k \left[\cV^{k+1}\right] =\left( 1-\frac {\tau}{n}\right) \cV^k + \frac{\tau}{n} \Ek\left[\left( f( x^{k+1}) -f(x^{\star })\right) ^{\frac {3}{2}}\right]	.
\end{equation}
	
\end{lemma}

\begin{proof}
The result follows immediately from the definition \eqref{eq:V_k} of $\cV^k$ and   the fact that $w_{i}^{k+1}=x^{k+1}$ with probability $\frac{\tau}{n}$ and $w_{i}^{k+1}=w_i^k$ with probability $1-\frac{\tau}{n}$.

\end{proof}

\subsection{Proof of Theorem~\ref{thm:cubic_new}}
\begin{proof}
Denote $C\eqdef  \frac{\sqrt{2}(M+H)}{3\mu^{3/2}}$. By combining the results from our lemmas, we can show
\begin{eqnarray*}
	\E_k\left[ \cV ^{k+1}\right] 
	%%%
	&\overset{\eqref{eq:nb97fg9vb98h8s}}{=} &  \left( 1-\frac {\tau}{n}\right) \cV^k + \frac{\tau}{n} \left( f( x^{k+1}) -f^{\star }\right) ^{\frac {3}{2}}\\
	%%%
	&\overset{\eqref{eq:JuI98Gt67Y}}{\leq} &  \left( 1-\frac {\tau}{n}\right) \cV^k + \frac{\tau}{n} \left( C \cV^k \right)^{\frac{3}{2}}\\
	&=&  \left[1 - \frac{\tau}{n} + \frac{\tau}{n}  C^{3/2} \sqrt{\cV^k} \right] \cV^k  \\
	&\overset{\eqref{eq:cubic_almost_sure}}{\le} & \left( 1-\frac {\tau}{2n}\right) \cV^k,
\end{eqnarray*}
where in the last step we also used the assumption on $f(w_i^0)-f(x^{\star})$ for $i=1,\dotsc,n$.
\end{proof}

\section{Efficient implementation for generalized linear models}
\begin{algorithm}[h]
\begin{algorithmic}
	\State \textbf{Initialize:} Choose starting iterates $w_1^0, w_2^0, \ldots, w_n^0 \in \R^d$, compute $\alpha_i^0 = \phi_i'(a_i^\top w_i^0)$, $\beta_i^0 = \phi_i''(a_i^\top w_i^0)$, $\gamma_i^0=a_i^\top w_i^0$ for $i=1,\dotsc, n$, $B^0=(\lambda I + \avein \beta_i^0 a_i a_i^\top)^{-1}$, $g^0=\avein \alpha_i^0 a_i$ $h^0=\avein \beta_i^0 \gamma_i^0 a_i$
	\For{$k = 0,1,2,\ldots$}
	\State $x^{k+1} = B^k(h^k- g^k) $
	\State Choose $i=i^k\in\{1,\dotsc, n\}$ uniformly at random
	\State $\alpha^{k+1}_{i}=\phi_i'(a_i^\top x^{k+1})$
	\State $\beta^{k+1}_{i}= \phi_i''(a_i^\top x^{k+1})$ 
	\State $\gamma^{k+1}_{i}=a_i^\top x^{k+1}$ 
	\State $g^{k+1}=g^k + (\alpha_{i}^{k+1} - \alpha_i^k)a_i$
	\State $h^{k+1}=h^k + (\beta_{i}^{k+1} \gamma_i^{k+1} - \beta_i^k\gamma_i^k)a_i$
	\State $B^{k+1}
	= B^k - \frac{\beta_i^{k+1}-\beta_i^k}{n+(\beta_i^{k+1}-\beta_i^k) a_i^\top B^k a_i}B^k a_i a_i^\top B^k$
	\EndFor
\end{algorithmic}
\caption{Stochastic   Newton (SN) for generalized linear models with $\tau=1$} \label{alg:newton_glm}
\end{algorithm}	

In this section, we consider fitting a generalized linear model (GLM) with $\ell_2$ regularization, i.e.\ solving
\begin{align*}
	\min_{x \in \R^d} \frac{1}{n}\sum_{i=1}^n \Bigl(\underbrace{\phi_i(a_i^\top x) + \frac{\lambda}{2}\|x\|^2}_{f_i(x)}\Bigr),
\end{align*}
where $\lambda>0$, $a_1, \dotsc, a_n\in\R^d$ are given and $\phi_1, \dotsc, \phi_n$ are some convex and smooth function. For any $x$, we therefore have $f_i(x) = \phi_i(a_i^\top x) + \frac{\lambda}{2}\|x\|^2$ and $\nabla^2 f_i(x)=\phi_i''(a_i^\top x)a_ia_i^\top + \lambda I$ for $i=1,\dotsc, n$.

Let us define for convenience
\begin{align*}
	H^k&\eqdef
	\lambda I + \avein \phi_i''(a_i^\top w_i^k) a_i a_i^\top,
	& B^k &\eqdef \left(H^k\right)^{-1}, \\
	g^k &\eqdef\avein \phi'(a_i^\top w_i^k)a_i, 
	& h^k &\eqdef \avein \phi_i''(a_i^\top w_i^k)a_i a_i^\top w_i^k.
\end{align*}
Then we can shorten the update rule for $x^{k+1}$ to $x^{k+1} = B^k (h^k + \lambda x^k  - (g^k + \lambda x^k))= B^k(h^k - g^k)$. Our key observation is that updates of $B^k$, $g^k$ and $h^k$ are very cheap both in terms of memory and compute.

First, it is evident that to update $g^k$ it is sufficient to maintain in memory only $\alpha_i^k \eqdef \phi_i'(a_i^\top w_i^k)$. Similarly, for $h^k$ we do not need to store the Hessians explicitly and can use instead $\beta_i^k \eqdef \phi_i''(a_i^\top w_i^k)$ and $\gamma_i^k \eqdef a_i^\top w_i^k$.
Then 
\begin{align*}
	H^{k+1}
	=H^k+\frac{1}{n}\sum_{i\in S^k}(\nabla^2 f_{i}(x^{k+1}) - \nabla^2 f_i(w_i^k))
	=H^k+\frac{1}{n}\sum_{i\in S^k}(\beta_i^{k+1}-\beta_i^k) a_i a_i^\top.
\end{align*}

The update above is low-rank and there exist efficient ways of computing the inverse of $H^{k+1}$ using the inverse of $H^k$. For instance, the following proposition is at the core of efficient implementation of quasi-Newton methods and it turns out to be quite useful for Algorithm~\ref{alg:newton} too.
\begin{proposition}[Sherman-Morrison formula]\label{pr:sherman}
	Let $H\in\R^{d\times d}$ be an invertible matrix and $u,v\in\R^d$ be arbitrary two vectors such that $1+u^\top H^{-1} v\neq 0$, then we have
	\begin{align*}
		(H+uv^\top)^{-1}
		= H^{-1} - \frac{1}{1+u^\top H^{-1} v}H^{-1}uv^\top H^{-1}.
	\end{align*}
\end{proposition}

	Consider for simplicity $\tau=|S^k|=1$, then $H^{k+1}=H^k+\frac{\beta_i^{k+1}-\beta_i^k}{n}a_i a_i^\top$ and by Proposition~\ref{pr:sherman}
\begin{align*}
	B^{k+1}
	= B^k - \frac{\beta_i^{k+1}-\beta_i^k}{n+(\beta_i^{k+1}-\beta_i^k) a_i^\top B^k a_i}B^k a_i a_i^\top B^k.
\end{align*}
 Thus, the complexity of each update of Algorithm~\ref{alg:newton_glm} is $\cO(d^2)$, which is much faster than solving linear system from the update of Newton method. 
\section{Efficient implementations of Algorithm~\ref{alg:cubic}}
The subproblem that we need to efficiently solve when running Algorithm~\ref{alg:cubic} can be reduced to
\begin{align}
	\min_{x \in \R^d} g^\top x + \frac{1}{2}x^\top H x + \frac{M}{6n}\sumin \|x - w_i\|^3, \label{eq:cubic_subproblem}
\end{align}
where $g\in \R^d, H\in\R^{d\times d}$ and $w_1,\dotsc, w_n\in\R^d$ are some given vectors. One way to solve it is to run gradient descent or a stochastic variance reduced algorithm. For instance, if $f_i(x) = \phi(a_i^\top x)+ \frac{\lambda}{2}\|x\|^2$, then $H$ is the sum of $n$ functions, whose individual gradients can be computed in $\cO(d)$, while the gradient of $\frac{1}{2}x^\top H x$ requires $\cO(d^2)$ flops.

A special care is required for the sum of cubic terms. Despite its simplicity, every cubic function $\|x - w_i\|^3$ is not globally smooth, so we can not rely on gradients to minimize it. In addition, it might become the bottleneck of the solver since processing it requires accessing $n$ vectors $w_1,\dotsc, w_n$, making solver's iterations too expensive.  We propose two ways of resolving this issue. 

First, one can use a stochastic method that is based on proximal operator. For instance, the method from~\cite{mishchenko2019stochastic} can combine gradient descent or SVRG~\cite{SVRG} with variance-reduced stochastic proximal point. Note that for a single cubic term, the proximal operator of $\|x-w\|^3$ can be computed in closed-form.

Second, one can slightly change the assumption in equation~\eqref{eq:taylor_inaccuracy_values} to use a different norm, $\|\cdot\|_3$ instead of $\|\cdot\|_2$. Due to norm equivalence, this does not actually change our assumptions, although it might affect the value of $M$. The advantage of using the new norm is that it allows us to simplify the sum of cubic terms as given below,
\begin{align*}
	\avein \|x - w_i\|_3^3
	= \|x\|_3^3 - 3\<x^2, \avein w_i> + 3\<x, \avein w_i^2> - \avein \|w_i\|_3^3,
\end{align*}
where by $x^2$ and $w_i^2$ we denoted the vectors whose entries are coordinate-wise squares of $x$ and $w_i$ correspondingly. Thus, to compute gradients and functional values of the cubic penalty we only need to access $\avein w_i$ and $\avein w_i^2$ rather than every $w_i$ individually. Furthermore, one can also compute proximal operator of this function in $\cO(d)$ flops, making it possible to run proximal gradient descent or its stochastic variant on the objective in~\eqref{eq:cubic_subproblem}.

\clearpage
\section{Experiments}

We run our experiments on $\ell_2$-regularized logistic regression problem with two datasets from the LIBSVM library: madelon and a8a. The objective is given by $f_i(x)=\log(1+\exp(-b_i a_i^\top x))+\frac{\lambda}{2}\|x\|^2$, where $a_1,\dotsc, a_n\in\R^d, b_1,\dotsc, b_n\in\R$ are given and  the value of $\lambda$ is either of $0, \frac{1}{10000n}$ and $\frac{1}{100n}$. 

We compared to  the standard (deterministic) Newton and incremental Newton from~\cite{rodomanov2016superlinearly}, all initialized at $x^0=0\in\R^d$. See Figure~\ref{fig:logistic_newton} for the results. Interestingly, we see that the incremental variant outperforms the full-batch Newton only by a slight margin and the stochastic method is even worse. All methods, however, solve the problem very fast even when it is ill-conditioned.

\begin{figure*}[h]
\begin{subfigure}[t]{0.32\textwidth}
	\includegraphics[width=1\textwidth]{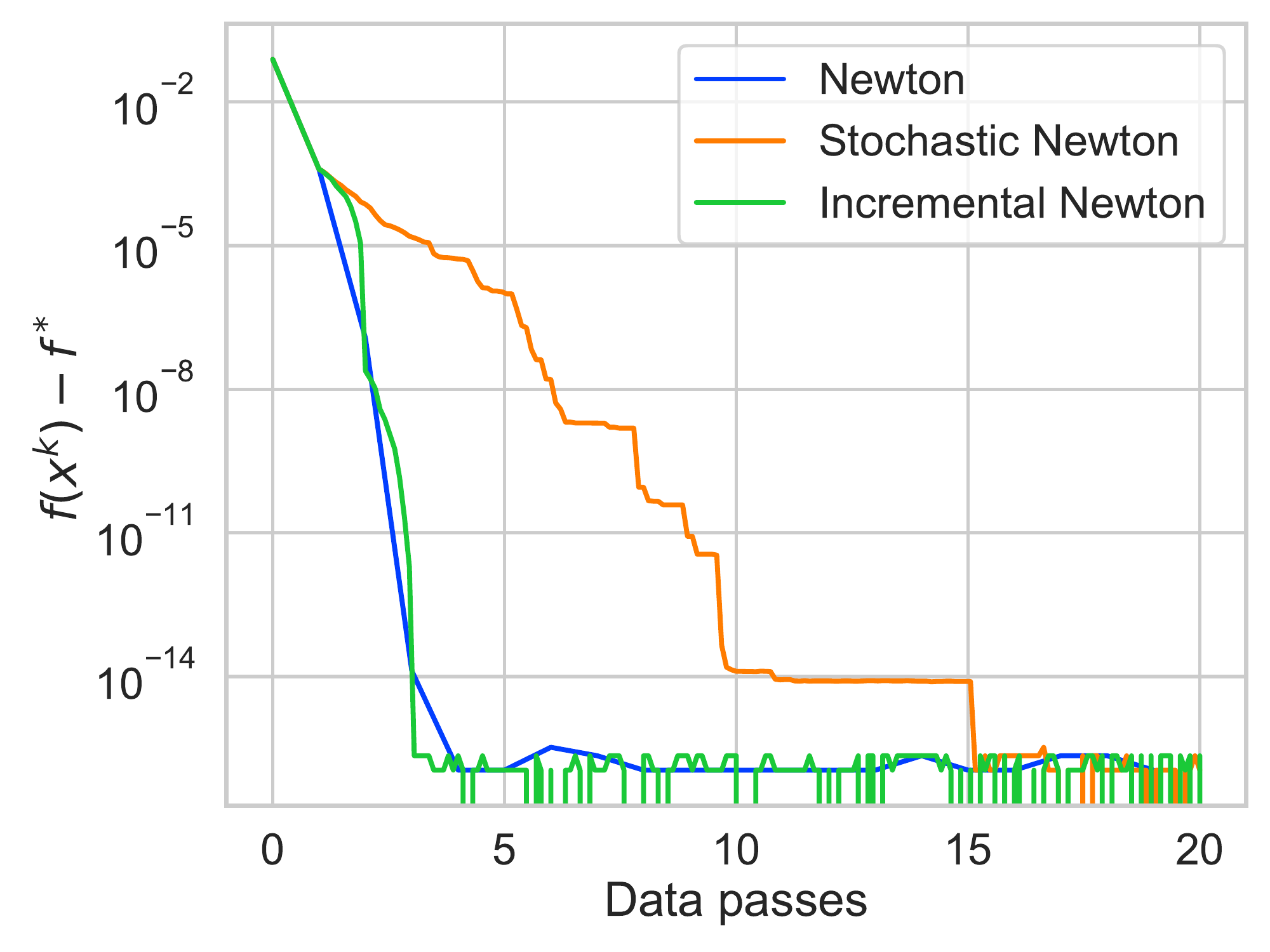}
	\caption{Madelon, $\lambda=\frac{1}{100n}$}
\end{subfigure}
\begin{subfigure}[t]{0.32\textwidth}
	\includegraphics[width=1\textwidth]{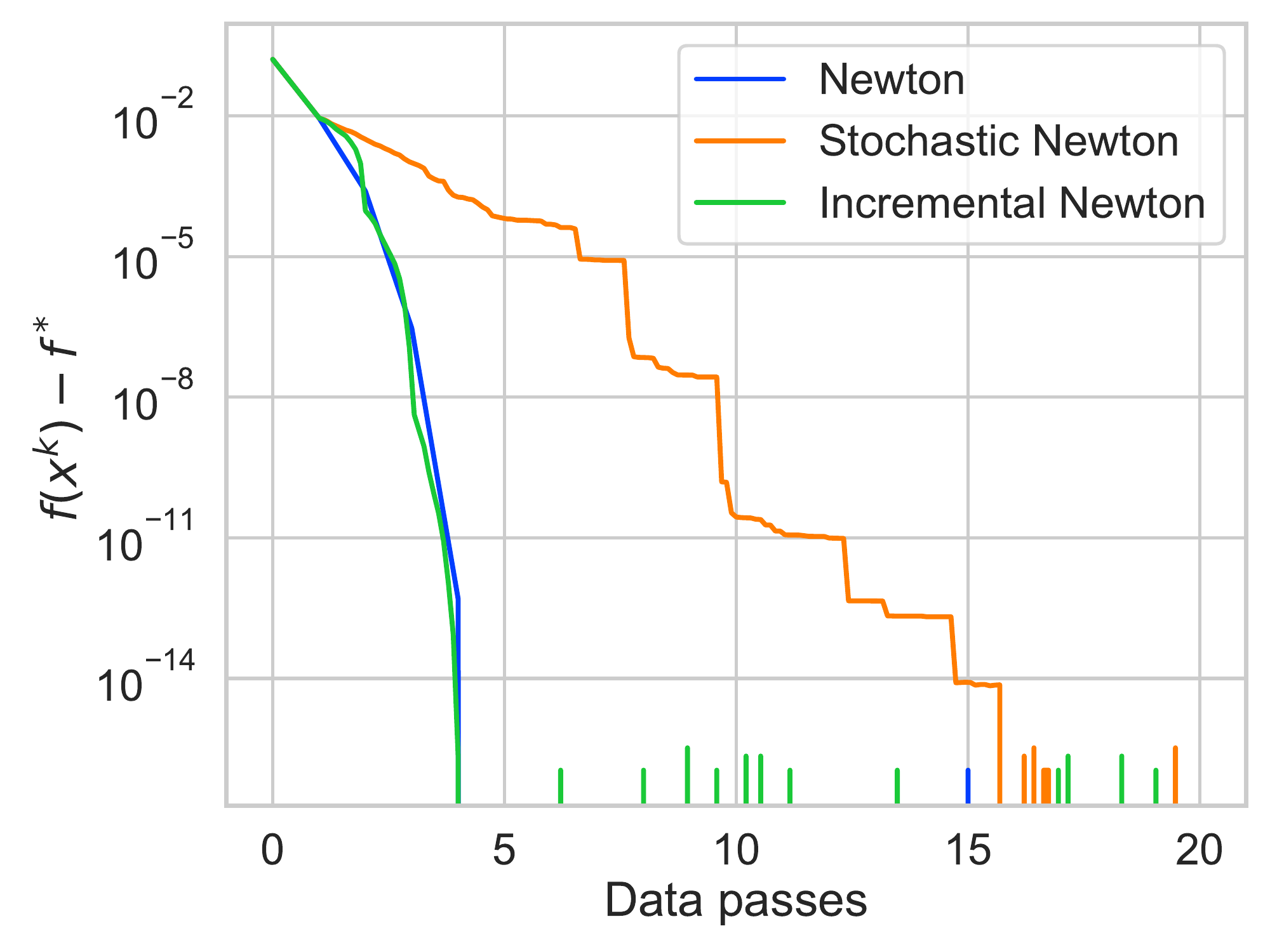}
	\caption{Madelon, $\lambda=\frac{1}{10000n}$}
\end{subfigure}
\begin{subfigure}[t]{0.32\textwidth}
	\includegraphics[width=1\textwidth]{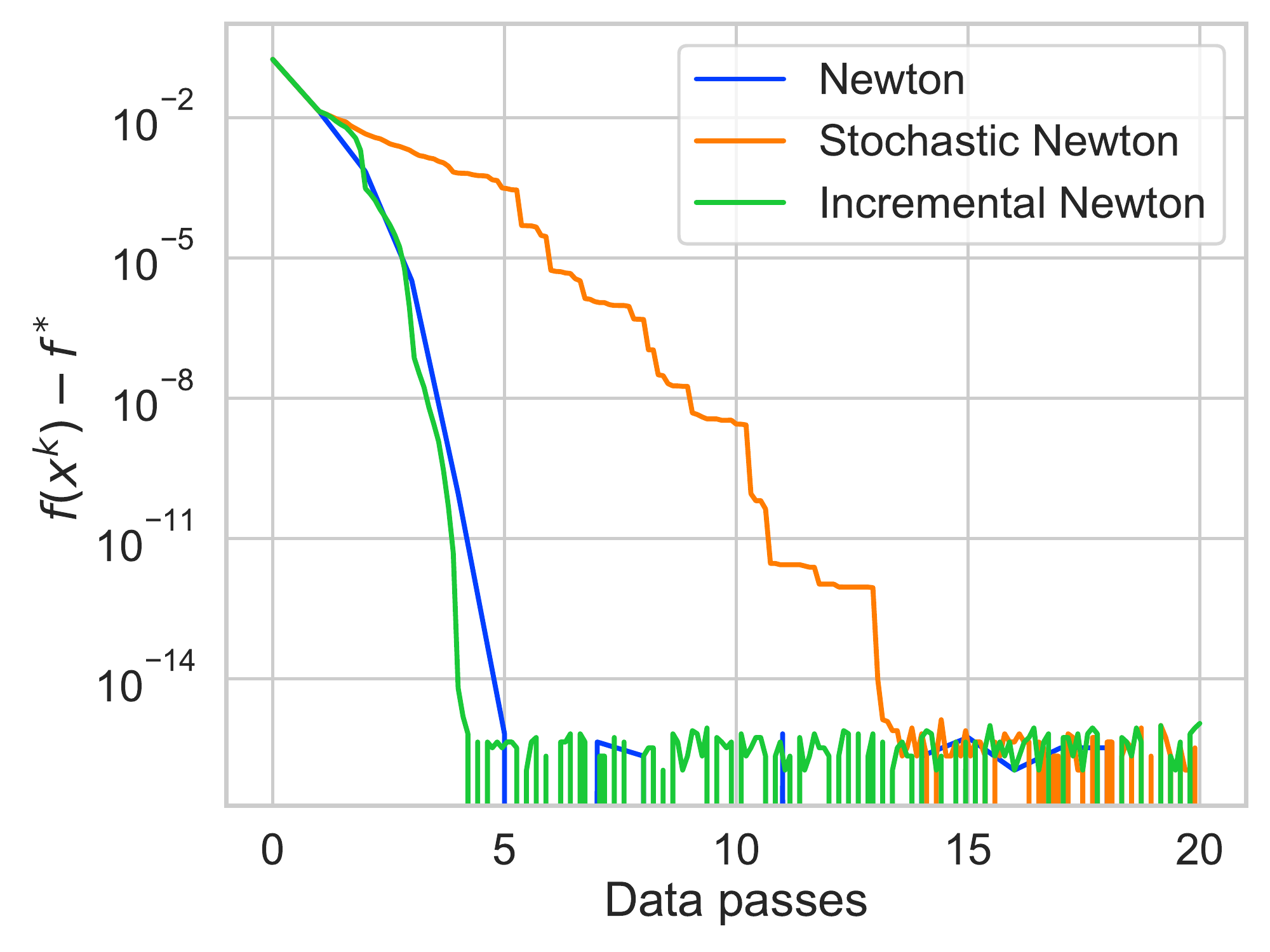}
	\caption{Madelon, $\lambda=0$}
\end{subfigure}\\
\begin{subfigure}[t]{0.32\textwidth}
	\includegraphics[width=1\textwidth]{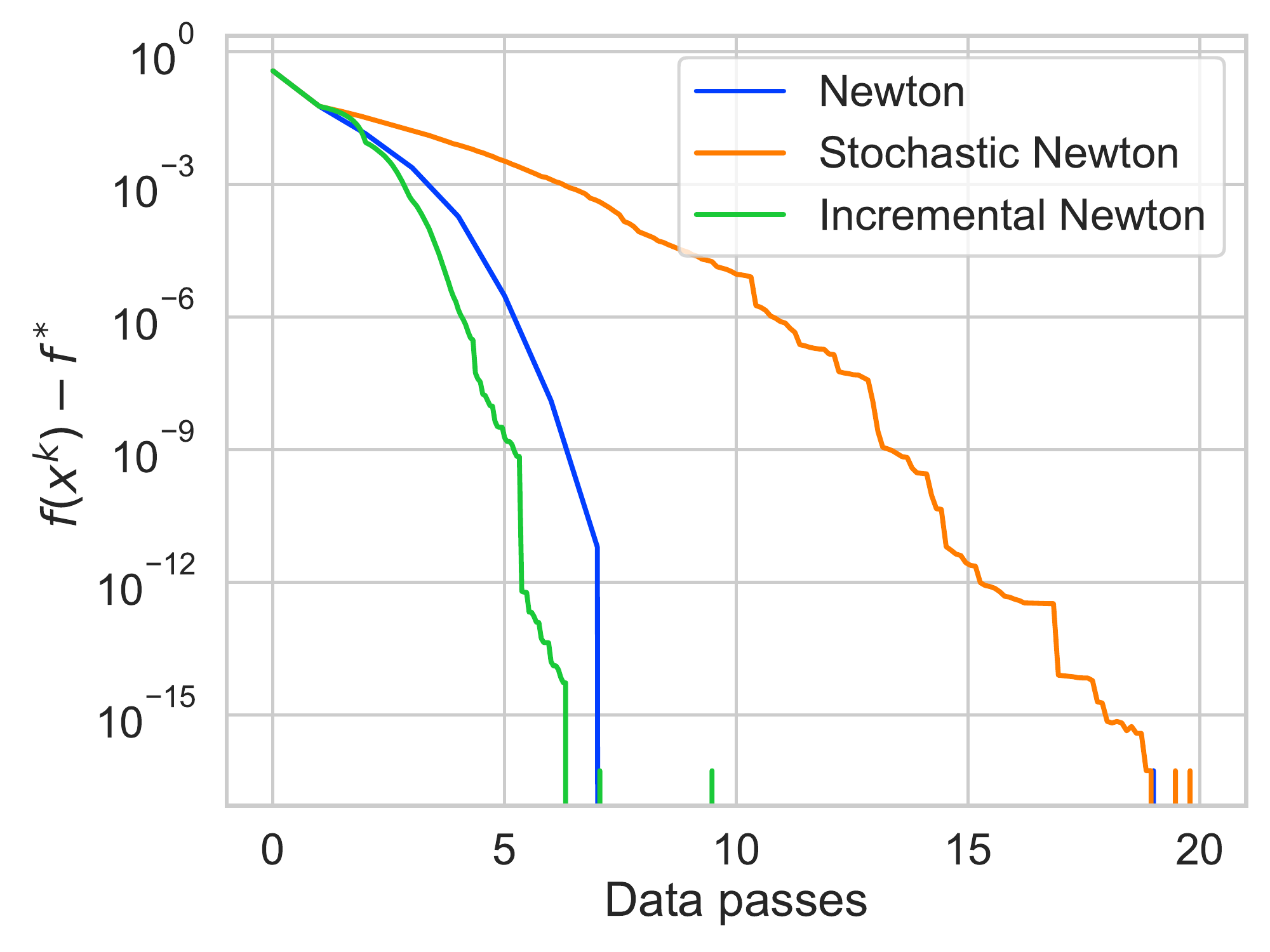}
	\caption{a8a, $\lambda=\frac{1}{100n}$}
\end{subfigure}
\begin{subfigure}[t]{0.32\textwidth}
	\includegraphics[width=1\textwidth]{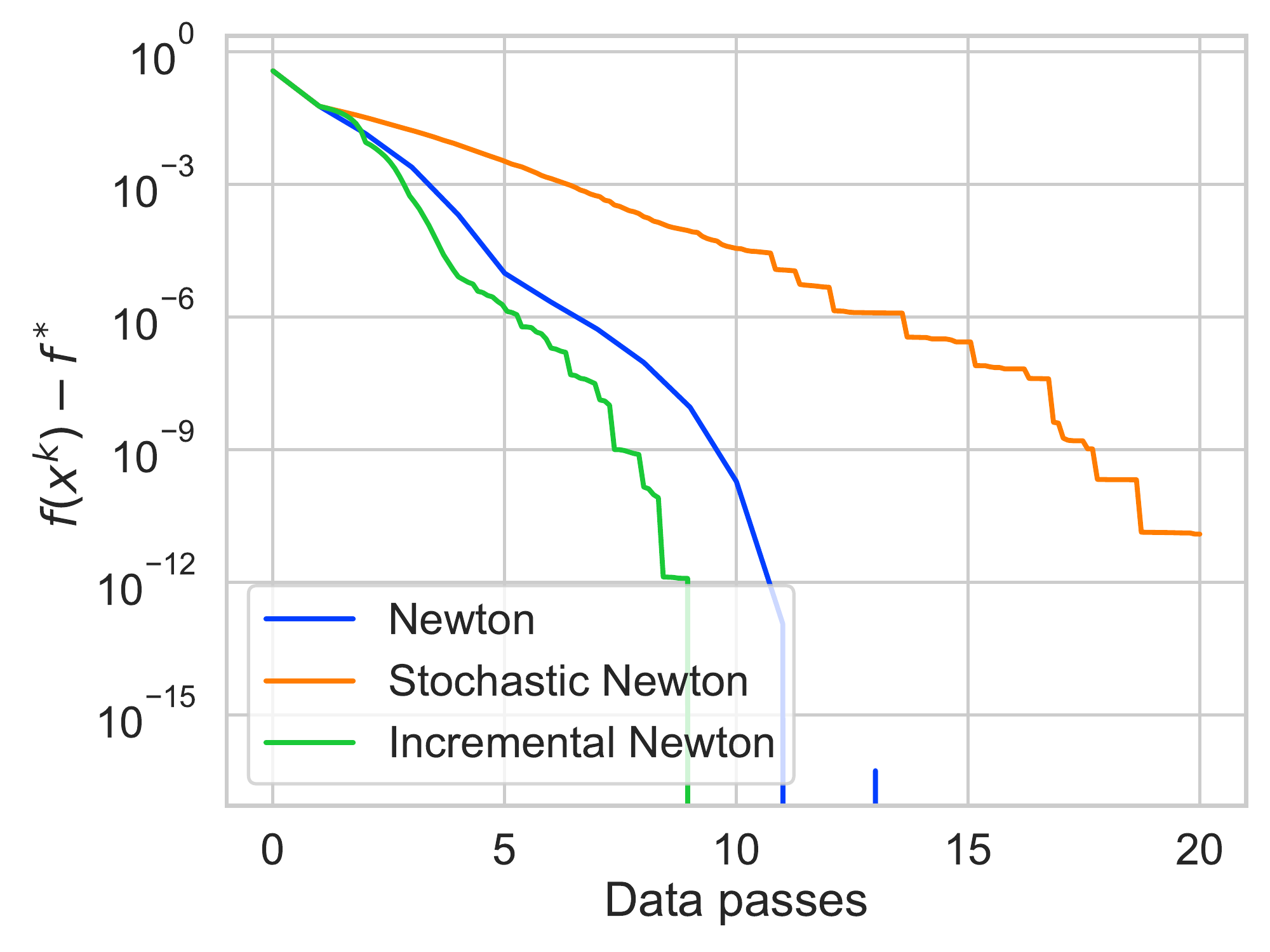}
	\caption{a8a, $\lambda=\frac{1}{10000n}$}
\end{subfigure}
\begin{subfigure}[t]{0.32\textwidth}
	\includegraphics[width=1\textwidth]{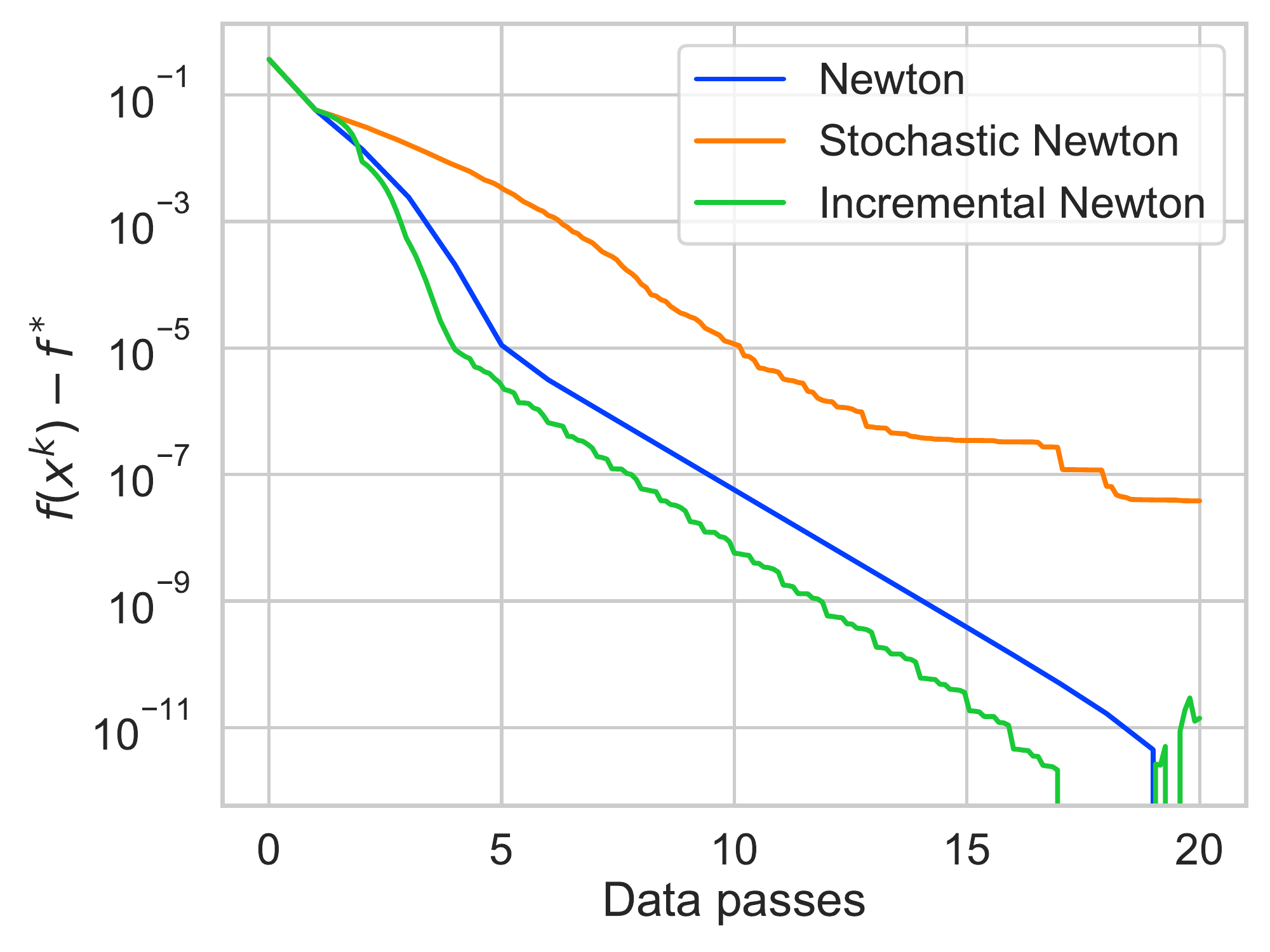}
	\caption{a8a, $\lambda=0$}
\end{subfigure}
\caption{Comparison of Newton methods when applied to logistic regression with $\ell_2$-regularization ($x$-axis is the number of full passes over the dataset and $y$-axis is the functional suboptimality).}
\label{fig:logistic_newton}
\end{figure*}

We also did a comparison of the full batch cubic Newton with our stochastic version. Since with initialization $x^0=0\in\R^d$ Newton algorithm itself is convergent, we chose an initialization much further from the optimum, namely $x^0=(0.5, \dotsc, 0.5)^\top$. With this choice of $x^0$, the value $M=0$ leads to divergence. Therefore, we separately tuned the optimal value of $M$ for both algorithms, deterministic and stochastic, and discovered that Algorithm~\ref{alg:cubic} converges under much smaller values of $M$ and its optimal value was typically close to the smallest under which it converges. This is similar to the behaviour of first-order approximations~\cite{qian2019miso}, where the quadratic penalty works under $n$ times slower coefficient.

\begin{figure*}[h]
\centering
\begin{subfigure}[t]{0.32\textwidth}
	\includegraphics[width=1\textwidth]{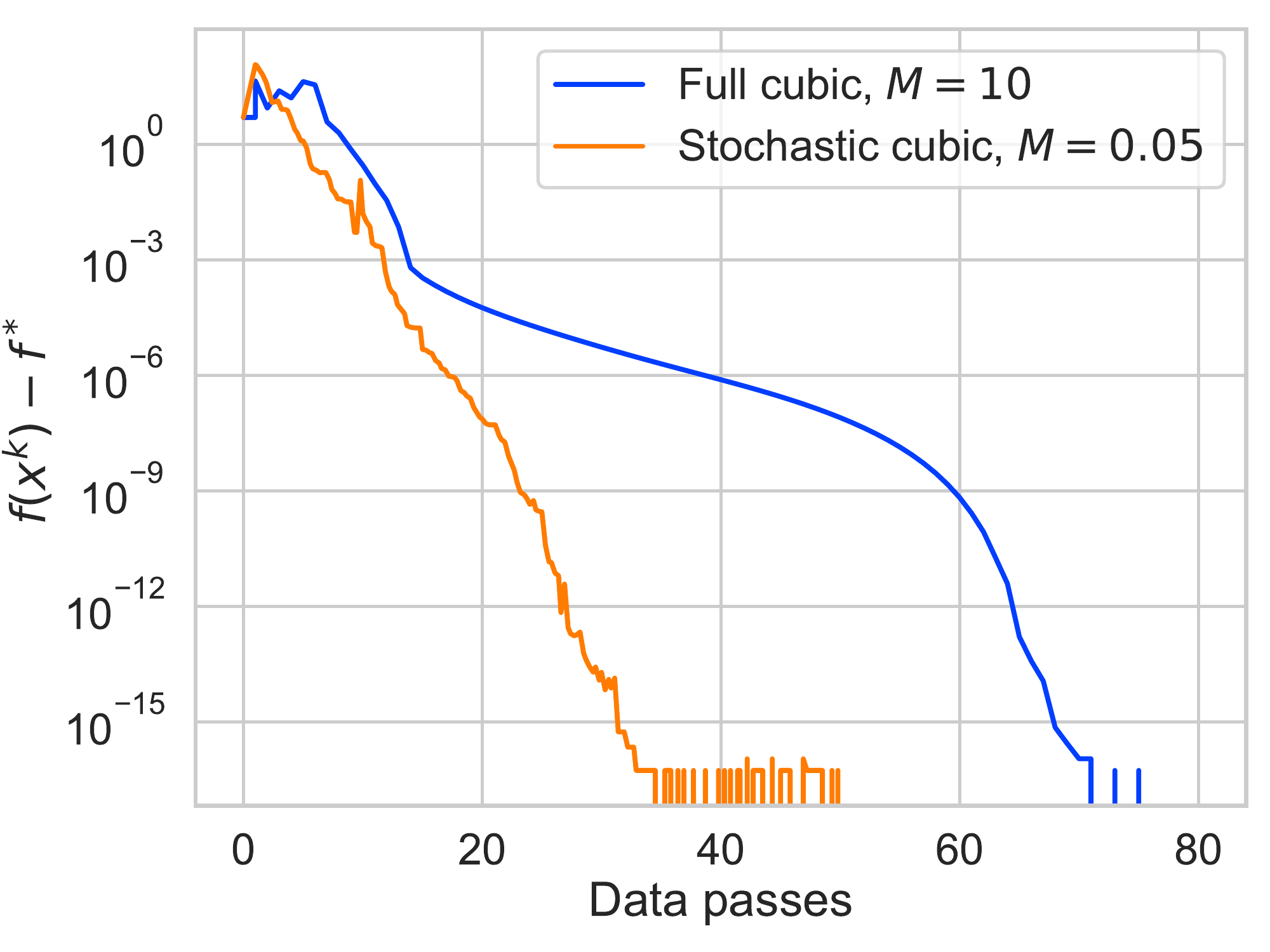}
	\caption{a8a, $\lambda=\frac{1}{100n}$}
\end{subfigure}
\begin{subfigure}[t]{0.32\textwidth}
	\includegraphics[width=1\textwidth]{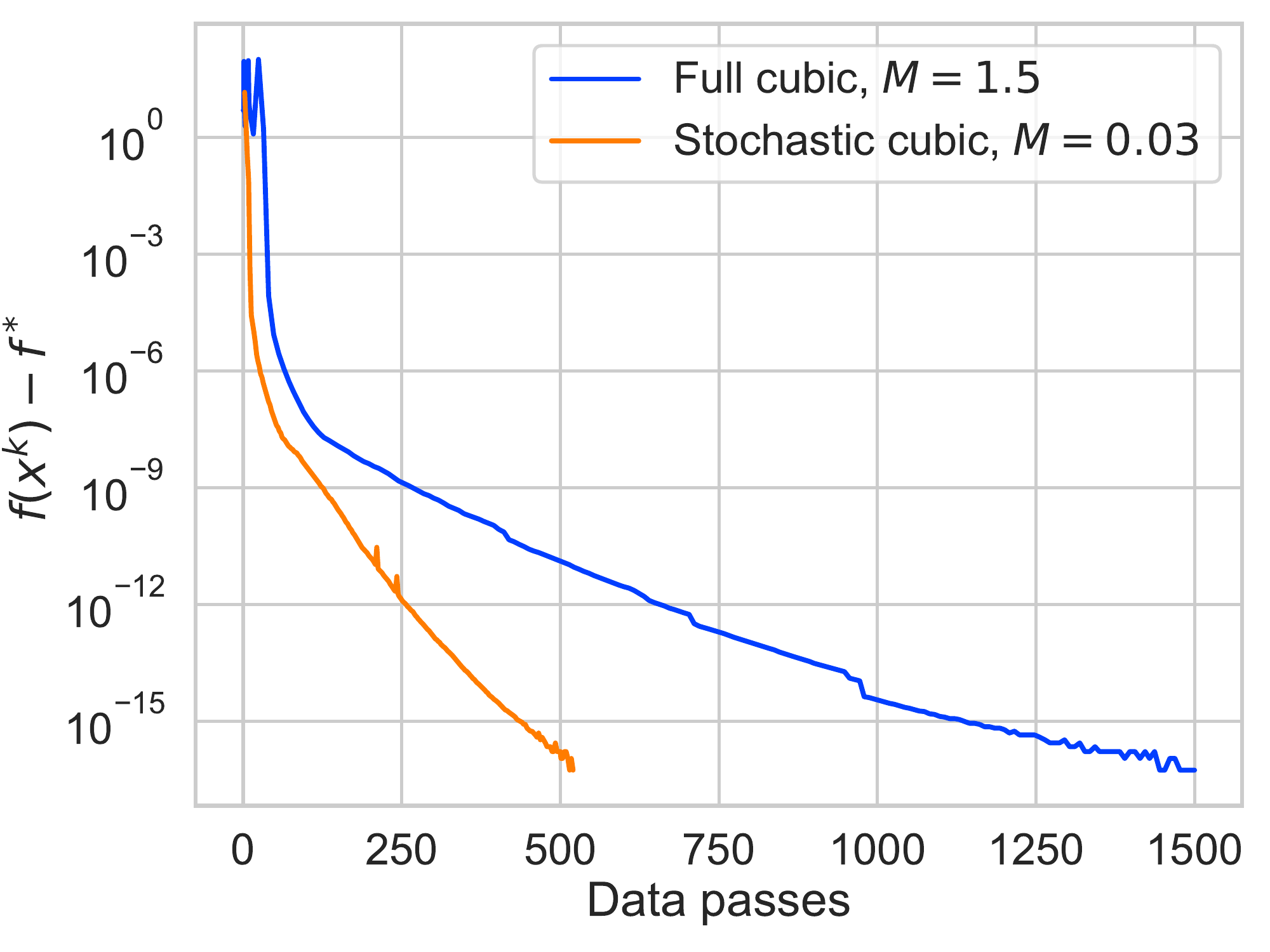}
	\caption{a8a, $\lambda=\frac{1}{10000n}$}
\end{subfigure}
\caption{Comparison of Newton methods with cubic regularization when applied to logistic regression with $\ell_2$-regularization ($x$-axis is the number of full passes over the dataset and $y$-axis is the functional suboptimality). The stochastic algorithm uses $n=10$.}
\label{fig:logistic_newton_cubic}
\end{figure*}

To solve the cubic subproblem~\eqref{eq:cubic_subproblem}, we used the algorithm from~\cite{mishchenko2019stochastic}. To get a finite-sum problem, we partitioned the a8a dataset into $n=10$ parts of approximately the same size. The results are presented in Figure~\ref{fig:logistic_newton_cubic}.

\end{document}